\documentclass[10pt]{article} 

\usepackage[accepted]{rlj}           

\usepackage{amssymb}            
\usepackage{mathtools}          
\usepackage{mathrsfs}           
\usepackage{graphicx}           
\usepackage{subcaption}         
\usepackage[space]{grffile}     
\usepackage{url}                
\usepackage{lipsum}             

\usepackage{amsmath}
\usepackage{amsthm}
\usepackage{verbatim}

\usepackage{multirow}

\usetikzlibrary{arrows.meta, positioning}

\tikzset{
    state/.style={
        circle, 
        draw, 
        thick, 
        minimum size=11mm,
        fill=gray!20,
        font=\large
    },
    startstate/.style={
        circle, 
        draw, 
        thick, 
        double, 
        minimum size=11mm,
        fill=gray!20,
        font=\large
    },
    lab/.style={font=\normalsize},
    >=Stealth
}

\title{Using Common Random Numbers for Simulation-based Planning with Rollouts}

\setrunningtitle{Using Common Random Numbers for Simulation-based Planning with Rollouts}

\author{Sandarbh Yadav\textsuperscript{1}, Frederic J Maliakkal\textsuperscript{1}, Harshad Khadilkar\textsuperscript{1}, Shivaram Kalyanakrishnan\textsuperscript{1}}

\emails{apexpacesandy@gmail.com, fredericjmaliakkal@gmail.com, shivaram@cse.iitb.ac.in, harshadk@iitb.ac.in}

\affiliations{
$^{1}$\textbf{Department of Computer Science and Engineering, IIT Bombay, India}
}

\contribution{
    We propose and analyse a new estimator of the value difference between policies that are identical beyond some time step.
    }
    {
    Existing estimators based on ``full independence'' and ``full dependence'' can exhibit higher variance on some tasks.
    }

\contribution{
    We demonstrate the algorithmic relevance of our proposed estimator and the existing ``full dependence'' estimator by highlighting their role in simulation-based planning with rollouts.
    }
    {
    Simulation-based planning with rollouts is a general technique, which is used in a variety of applications.
    }


\summary{Simulation-based planning with rollouts is a widely-deployed technique for decision making in stochastic environments. The primary instrument of simulation-based planning is a \textit{sampling model}, which is repeatedly called to generate trajectories and estimate the utilities of available actions. Among the actions thus explored, one with the maximum estimated utility is then executed. In this paper, we examine the effect of using \textit{common random numbers} in the simulation process. We obtain a simple recipe for (provably) reducing variance in relative utility when simulations invoke a rollout policy beyond some depth. Experiments on synthetic tasks confirm that our scheme improves task performance. The broader significance of our innovation is apparent from two practical applications: (1) single-step lookahead planning in a pension-disbursement task, and (2) a deployment of the well-known UCT algorithm for the game of Ludo.
}

\newcommand{\I}{\text{I}}
\newcommand{\D}{\text{D}}
\newcommand{\DD}{\text{DD}}

\newcommand{\blank}{\phantom{$\times$}}
\newcommand{\full}{$\times$}

\newcommand{\var}{\mathrm{var}}
\newcommand{\cov}{\mathrm{cov}}

\newtheorem{theorem}{Theorem}

\newtheorem{proposition}[theorem]{Proposition}

\begin{document}

\makeCover  
\maketitle  

\begin{abstract}

Simulation-based planning with rollouts is a widely-deployed technique for decision making in stochastic environments. The primary instrument of simulation-based planning is a \textit{sampling model}, which is repeatedly called to generate trajectories and estimate the utilities of available actions. Among the actions thus explored, one with the maximum estimated utility is then executed. In this paper, we examine the effect of using \textit{common random numbers} in the simulation process. We obtain a simple recipe for (provably) reducing variance in relative utility when simulations invoke a rollout policy beyond some depth. Experiments on synthetic tasks confirm that our scheme improves task performance. The broader significance of our innovation is apparent from two practical applications: (1) single-step lookahead planning in a pension-disbursement task, and (2) a deployment of the well-known UCT algorithm for the game of Ludo.
\end{abstract}

\section{Estimating Value Difference}
\label{sec:estimatingvaluedifference}

The essence of this paper lies in a mathematical problem related to statistical estimation. We motivate and formalise the problem in this introductory section. The subsequent sections present our solution to the problem and demonstrate its relevance to simulation-based planning.

Consider a finite-horizon Markov Decision Problem (MDP) $M$, comprising a set of states $S$ and a set of actions $A$. Let $H \geq 1$ be the horizon; we denote as $[H]$ the set $\{1, 2, \dots, H\}$. For $t \in [H]$, taking action $a \in A$ from state $s \in S$ at time step $t$ yields an immediate reward $R(s, a, t)$, and transitions to a random state $s^{\prime}$ for time step $t + 1$. For $t \in [H] \setminus \{H\}$, next state $s^{\prime}$ is in $S$, and is reached with probability $P(s, a, t, s^{\prime})$. The episode ends after $H$ steps; a convenient view is that all transitions from time step $H$ lead to a special terminal state $s_{\top}$. To reduce notational clutter, we assume that our MDP $M$ has a single starting state $s_{1} \in S$, and also that returns are undiscounted (generalising to a distribution of starting states and a discount factor is relatively straightforward). Hence, $M$ is fully specified by the tuple $(S, A, P, R, H, s_{1})$.

For a given policy $\pi: S \times [H] \to A$, the value function $V^{\pi}_{M}: S \times [H] \to \mathbb{R}$ is defined by
\begin{align}
\label{eq:bellman}
V^{\pi}_{M}(s, t) = R(s, \pi(s, t), t) + \sum_{s^{\prime} \in S \cup \{s_{\top}\}} P(s, \pi(s, t), t, s^{\prime}) V^{\pi}_{M}(s^{\prime}, t + 1)
\end{align}
for $s \in S$ and $t \in [H]$, with the convention that $P(s, \pi(s, H), H, s_{\top}) = 1$ and $V^{\pi}_{M}(s_{\top}, H + 1) = 0$. Thus, $V^{\pi}_{M}(s, t)$ is the expected return from $\pi$ if starting at $s$, with $H - t$ time steps left. Of special interest to us is $V^{\pi}_{M}(s_{1}, 1)$, which we take as the utility of $\pi$ on $M$, and denote $U(\pi, M)$.

$U(\pi, M)$ can be computed by dynamic programming (starting with $t = 1$) if $S$ is not too large, and importantly, if $R$ and $P$ are available as tables or simple functions. Our interest is in the relatively common situation that $P$ \textit{cannot} be accessed as a distribution; however, information about $P$ can still be obtained by \textit{sampling} it~\citep[see Chapter 8]{sutton2018reinforcement}. In other words, for $(s, a, t) \in S \times A \times [H]$, we can sample a next state $s^{\prime}$ from the distribution $P(s, a, t, \cdot)$. We assume that $R$ is known exactly, as is common in practice~\citep{A.Ng2003, silver2016mastering,dam2022monte}.

\subsection{Estimating Value}
\label{subsec:estimatingvaluedifference-estimatingvalue}

Now, consider the task of obtaining an estimate $E$ of $U(\pi, M)$ by sampling $P$. The most natural way to proceed would be to begin at $s_{1}$, and for $t = 1, 2, \dots, H$, generate $s_{t + 1} \sim P(s_{t}, \pi(s_{t}, t), t, \cdot)$. Setting $E$ to be the sum of the rewards accumulated along this trajectory---that is, $\sum_{t = 1}^{H} R(s_{t}, \pi(s_{t}, t), t)$---would make it an unbiased estimate of $U(\pi, M)$. Call this the ``forward process'' to obtain the estimate $E$.

For analytical purposes, we may also view $E$ as the outcome of a ``backward process'', in which we first \textit{sample} a deterministic MDP $M^{\prime} = (S, A, P^{\prime}, R, H, s_{1})$ from $M$ (we denote this operation $M^{\prime} \sim M$). $M^{\prime}$ is obtained by sampling a single state $\mathrm{next}(s, a, t) \sim P(s, a, t, \cdot)$ for each $(s, a, t) \in S \times A \times [H]$, and setting for all $s^{\prime} \in S$:
$$P^{\prime}(s, a, t,  s^{\prime}) = \begin{cases} 1 & \text{if } s^{\prime} = \mathrm{next}(s, a, t), \\ 0 & \text{otherwise.} \\ \end{cases}$$
Notice that $P^{\prime}(s, a, t, s^{\prime})$ is $1$ with probability $P(s, a, t, s^{\prime})$. With $P^{\prime}$ so defined, our estimate $E$ is nothing but the utility of $\pi$ on $M^{\prime}$: that is, $E = U(\pi, M^{\prime})$. Since our process ensures that $\mathbb{E}[P^{\prime}(s, a, t, s^{\prime})] = P(s, a, t, s^{\prime})$, it follows that
\begin{align}
\underset{M^{\prime} \sim M}{\mathbb{E}}[V^{\pi}_{M^{\prime}}(s, t)]
&= \underset{M^{\prime} \sim M}{\mathbb{E}}[R(s, \pi(s), t)] + \sum_{s^{\prime} \in S} \underset{M^{\prime} \sim M}{\mathbb{E}}[P^{\prime}(s, \pi(s), t, s^{\prime}) V^{\pi}_{M^{\prime}}(s^{\prime}, t + 1)] \nonumber \\
&= R(s, \pi(s), t) + \sum_{s^{\prime} \in S} \underset{M^{\prime} \sim M}{\mathbb{E}}[P^{\prime}(s, \pi(s), t, s^{\prime})] \underset{M^{\prime} \sim M}{\mathbb{E}} [V^{\pi}_{M^{\prime}}(s^{\prime}, t + 1)]
\label{eq:bellmanmprime}
\end{align}
since $P^{\prime}(s, \pi(s), t, s^{\prime})$ and $V^{\pi}_{M^{\prime}}(s^{\prime}, t + 1)$ are independent. Beginning with the base case of $V^{\pi}_{M^{\prime}}(s_{\top}, H + 1) = 0$, an inductive argument that compares the right hand sides of \eqref{eq:bellman} and \eqref{eq:bellmanmprime} lets us replace $\mathbb{E}_{M^{\prime} \sim M} [V^{\pi}_{M^{\prime}}(s^{\prime}, t + 1)]$ by $V^{\pi}_{M}(s^{\prime}, t + 1)$, and thereafter replace $\mathbb{E}_{M^{\prime} \sim M}[P^{\prime}(s, \pi(s), t, s^{\prime})]$ by $P(s, \pi(s), t, s^{\prime})$, eventually confirming that $\mathbb{E}_{M^{\prime} \sim M}[E] = \mathbb{E}_{M^{\prime} \sim M}[U(\pi, M^{\prime})] = U(\pi, M)$.

\subsection{Two Estimators of Value Difference}
\label{subsec:estimatingvaluedifference-twoestimatorsofvaluedifference}

Our main purpose for value estimation is to inform \textit{control}: that is, to help select a maximising action or policy. To focus, for now, on the key question that arises, we restrict ourselves to \textit{two} given policies, $\pi_{1}$ and $\pi_{2}$. Our aim is to furnish an estimator of $U(\pi_{1}, M) - U(\pi_{2}, M)$. How to do so?

The value estimation technique from Section~\ref{subsec:estimatingvaluedifference-estimatingvalue} provides a ready recipe: we can \textit{independently} estimate  $U(\pi_{1}, M)$ and $U(\pi_{2}, M)$ and then supply the difference. Indeed, let $X_{\I}$ denote the resulting estimator (``I'' for ``independent''); below, we specify the process for generating it.
\begin{center}
\framebox{\parbox[c]{0.83\textwidth}{
\noindent\underline{$X_{\I}(M, \pi_{1}, \pi_{2})$.} Sample $M_{1} \sim M$ and $M_{2} \sim M$. Return $U(\pi_{1}, M_{1}) - U(\pi_{2}, M_{2})$.
}}
\end{center}
It is clear that $X_{\I}$ is an unbiased estimator of $U(\pi_{1}, M) - U(\pi_{2}, M)$. In practice, we will typically draw multiple i.i.d.~samples of $X_{\I}$ and compare the average with $0$ to determine whether $\pi_{1}$ or $\pi_{2}$ is superior. In turn, the number of samples we need to reach any fixed level of confidence will scale with the \textit{variance} of $X_{\I}$. Qualitatively, what factors contribute to this \textit{variance}? Naturally, the intrinsic variance of the returns of $\pi_{1}$ and $\pi_{2}$ on $M$ have a role to play. On the other hand, notice that $\pi_{1}$ and $\pi_{2}$ have been evaluated on \textit{independent} samples of $M$---which means $X_{\I}$ will generally have variance on stochastic MDPs even when $\pi_{1} = \pi_{2}$. Can we revise our estimator to eliminate this second component of the variance? Indeed, the founding principle of variance reduction in MDPs~\citep{kearns1999approximate, Ng2000PEGASUSAP}
is to \textit{fix} the random MDP on which the candidate policies are evaluated. In our simplified setting, this amounts to defining a second estimator $X_{\D}$ (``D'' for ``dependent), as below.
\begin{center}
\framebox{\parbox[c]{0.7\textwidth}{
\noindent\underline{$X_{\D}(M, \pi_{1}, \pi_{2})$.} Sample $M_{1} \sim M$. Return $U(\pi_{1}, M_{1}) - U(\pi_{2}, M_{1}).$
}}
\end{center}
$X_{\D}$ is again an unbiased estimator, but are we sure that it will have lower variance than $X_{\I}$? The starting point of this paper is the somewhat surprising \textit{negative} result.
\begin{proposition}
\label{prop:negative}
There exist MDP $M$ and corresponding policies $\pi_{1}, \pi_{2}$ such that $$\var(X_{D}(M, \pi_{1}, \pi_{2})) > \var(X_{I}(M, \pi_{1}, \pi_{2})).$$
\end{proposition}
\begin{proof}
We apply the identity ``$\var(Y - Z) = \var(Y) + \var(Z) - 2 \cov (Y, Z)$'', to observe
\begin{align}
\var\left(X_{\D}(M, \pi_{1}, \pi_{2})\right) &= \underset{M_{1} \sim M}{\var} \left(U(\pi_{1}, M_{1}) \right) + \underset{M_{1} \sim M}{\var} \left( U(\pi_{2}, M_{1}) \right) \nonumber \\
& \phantom{aaaa} - 2 \underset{M_{1} \sim M}{\cov} \left( U(\pi_{1}, M_{1}), U(\pi_{2}, M_{1}) \right); \label{eq:varxd}\\
\var\left(X_{\I}(M, \pi_{1}, \pi_{2})\right) &= \underset{M_{1} \sim M}{\var} \left( U(\pi_{1}, M_{1}) \right) + \underset{M_{2} \sim M}{\var} 
\left( U(\pi_{2}, M_{2})\right) \nonumber \\
& \phantom{aaaa}  - 2 \underset{M_{1}, M_{2} \sim M}{\cov} \left( U(\pi_{1}, M_{1}), U(\pi_{2}, M_{2}) \right). \label{eq:varxi}
\end{align}
The variance terms on the right-hand sides of \eqref{eq:varxd} and \eqref{eq:varxi} are identical. Also, $U(\pi_{1}, M_{1})$ and $U(\pi_{2}, M_{2})$ are \textit{independent}---making their covariance $0$. To achieve the proof, it suffices to furnish an MDP $M$ and corresponding policies $\pi_{1}, \pi_{2}$ such that $\cov_{M_{1} \sim M} \left( U(\pi_{1}, M_{1}), U(\pi_{2}, M_{1}) \right) < 0$.

Consider the 2-step MDP $M$ from Figure~\ref{fig:counterexample}, which has a minimal set of states and actions to establish our claim. Upon taking action $0$ from the starting state $s_{1}$, the agent transits uniformly at random to $s_{2}$ or $s_{3}$, receiving no reward. From states $s_{2}$ and $s_{3}$ at time step $t = 2$, actions $0$ and $1$ both lead to termination, with respective rewards $r_{0}$, $r_{1}$, $r_{2}$, and $r_{3}$. Our policies $\pi_{1}$ and $\pi_{2}$ vary in their actions at $s_{2}$ and $s_{3}$. They satisfy $\pi_{1}(s_{1}, 1) = 0$, $\pi_{1}(s_{2}, 2) = 0$, $\pi_{1}(s_{3}, 2) = 0$, $\pi_{2}(s_{1}, 1) = 0$, $\pi_{2}(s_{2}, 2) = 1$, $\pi_{2}(s_{3}, 2) = 1$. With this setup, we obtain
\begin{align*}
& \underset{M_{1} \sim M}{\cov} \left( U(\pi_{1}, M_{1}), U(\pi_{2}, M_{1}) \right) \\   
&= \underset{M_{1} \sim M}{\mathbb{E}} \left[ U(\pi_{1}, M_{1}) U(\pi_{2}, M_{1}) \right] - 
\underset{M_{1} \sim M}{\mathbb{E}} \left[ U(\pi_{1}, M_{1}) \right] \underset{M_{1} \sim M}{\mathbb{E}} \left[ U(\pi_{2}, M_{1}) \right] \\
&= \left(\frac{r_{0} r_{1}}{2}  + \frac{r_{2}r_{3}}{2}  \right) - \left( \frac{r_{0}}{2}  + \frac{r_{2}}{2}  \right) \left( \frac{r_{1}}{2} + \frac{r_{3}}{2} \right)
= \frac{(r_{0} - r_{2})(r_{1} - r_{3})}{4}.
\end{align*}
It is apparent that $r_{0}, r_{1}, r_{2}, r_{3}$ can be set to make the expression positive, negative, or zero.
\end{proof}

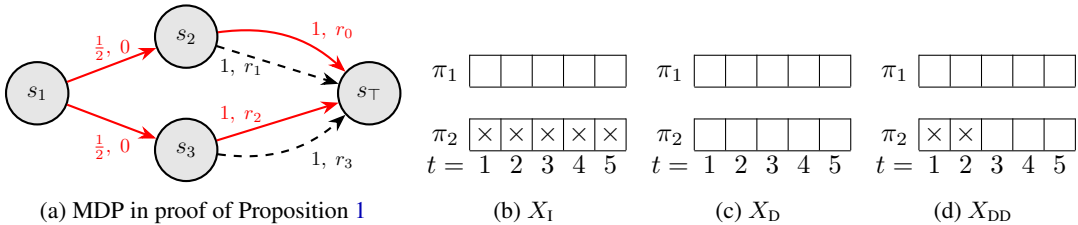
\begin{figure}[b]
    \centering
\subfloat[MDP in proof of Proposition~\ref{prop:negative}]{
\begin{tikzpicture}[scale=0.75, transform shape, node distance=12mm and 18mm]


\node[state] (s0) {$s_1$};
\node[state, right=of s0, yshift=10mm, xshift=-3mm] (s1) {$s_2$};
\node[state, right=of s0, yshift=-10mm, xshift=-3mm] (s2) {$s_3$};
\node[state, right=of s1, xshift=3mm, yshift=-10mm] (s3) {$s_{\top}$};


\draw[->, thick, red]
(s0) --
node[lab, pos=0.5, yshift=3mm] {$\frac{1}{2},\;0$}
(s1);

\draw[->, thick, red]
(s0) --
node[lab, pos=0.5, yshift=-4.5mm] {$\frac{1}{2},\;0$}
(s2);


\draw[->, thick, black, dashed]
(s1) --
node[lab, pos=0.2, yshift=-2.5mm] {$1,\;r_{1}$}
(s3);

\draw[->, thick, red]
(s1) to[bend left=25]
node[lab, pos=0.5, right=3mm, yshift=1mm] {$1,\;r_{0}$}
(s3);


\draw[->, thick, black, red]
(s2) --
node[lab, pos=0.2, yshift=3.2mm] {$1,\;r_{2}$}
(s3);

\draw[->, thick, black, black, dashed]
(s2) to[bend right=25]
node[lab, pos=0.5, right=3mm, yshift=-1.8mm] {$1,\;r_{3}$}
(s3);

\end{tikzpicture}
\label{fig:counterexample}
}
    \subfloat[$X_{\I}$]{%
{\setlength{\tabcolsep}{2pt}
\begin{tabular}{c|c|c|c|c|c|}
\cline{2-6}
$\pi_{1}$ & \blank & \blank & \blank & \blank & \blank \\
\cline{2-6}
\multicolumn{6}{c}{~} \\
\cline{2-6}
$\pi_{2}$ & \full & \full & \full & \full & \full\\
\cline{2-6} 
\multicolumn{1}{c}{$t =$}
& \multicolumn{1}{c}{$1$} 
& \multicolumn{1}{c}{$2$} 
& \multicolumn{1}{c}{$3$} 
& \multicolumn{1}{c}{$4$}
& \multicolumn{1}{c}{$5$} \\ 
\end{tabular}
}
        \label{fig:xi}%
    }
    \subfloat[$X_{\D}$]{%
{\setlength{\tabcolsep}{2pt}
\begin{tabular}{c|c|c|c|c|c|}
\cline{2-6}
$\pi_{1}$ & \blank & \blank & \blank & \blank & \blank \\
\cline{2-6}
\multicolumn{6}{c}{~} \\
\cline{2-6}
$\pi_{2}$ & \blank & \blank & \blank & \blank & \blank\\
\cline{2-6}
\multicolumn{1}{c}{$t =$}
& \multicolumn{1}{c}{$1$} 
& \multicolumn{1}{c}{$2$} 
& \multicolumn{1}{c}{$3$} 
& \multicolumn{1}{c}{$4$}
& \multicolumn{1}{c}{$5$} \\ 
\end{tabular}
}
        \label{fig:xd}%
    }
    \subfloat[$X_{\DD}$]{%
{\setlength{\tabcolsep}{2pt}
\begin{tabular}{c|c|c|c|c|c|}
\cline{2-6}
$\pi_{1}$ & \blank & \blank & \blank & \blank & \blank \\
\cline{2-6}
\multicolumn{6}{c}{~} \\
\cline{2-6}
$\pi_{2}$ & \full & \full & \blank & \blank & \blank\\
\cline{2-6}
\multicolumn{1}{c}{$t =$}
& \multicolumn{1}{c}{$1$} 
& \multicolumn{1}{c}{$2$} 
& \multicolumn{1}{c}{$3$} 
& \multicolumn{1}{c}{$4$}
& \multicolumn{1}{c}{$5$} \\ 
\end{tabular}
}
        \label{fig:xdd}%
    }
\caption{Figure (a) shows the MDP defined in the proof of Proposition~\ref{prop:negative}; each transition is labeled with ``probability, reward''. Figures (b), (c), and (d) pictorially illustrate the sampling process under $X_{\I}$, $X_{\D}$, and $X_{\DD}$ (introduced in Section~\ref{sec:depth-baseddependenceforrollouts}), respectively. The figures depict a task with horizon 5, and assume $\pi_{1}$ and $\pi_{2}$ are identical for steps $t = 3, 4, 5$ (that is, they \textit{agree} after 2 steps). Each small square is a time step. If corresponding squares for the policies are blank, then both policies use the same transition sample for each (state, action) pair at that time step. When one square has a cross, it means the policies are evaluated on different (independently generated) samples for that time step.}
\label{fig:sec1figures}
\end{figure}

The MDP used in the proof can be expanded to have any number of states and actions, and also a longer horizon, while preserving the variances of $X_{\D}$ and $X_{\I}$. Nonetheless, the small MDP itself conveys the intuition behind how $X_{\D}$ can get to have larger variance than $X_{\I}$. Notice that for this to happen---that is, for 
$\cov_{M_{1} \sim M} \left( U(\pi_{1}, M_{1}), U(\pi_{2}, M_{1}) \right)$ to be negative---we must have  different signs for $(r_{0} - r_{2})$ and $(r_{1} - r_{3})$. The intuition is that if $\pi_{1}$ stands to benefit by starting at $s_{2}$ instead of $s_{3}$, then $\pi_{2}$ stands to lose by starting at $s_{2}$ instead of $s_{3}$. 
This sort of structure is not common in practical tasks---typically, the order between two states' values stays the same for different policies. In our upcoming experiments on real-world tasks, $X_{\D}$ invariably delivers better results than $X_{\I}$. 

Our investigation in this section leaves us with a question to resolve: \textit{is there an unbiased estimator of value difference, which makes no more calls to the simulator than $X_{\I}$, but which is guaranteed to have a smaller variance than $X_{\I}$ on \textit{all} MDPs?} We obtain a positive answer for a special class of policies, which, interestingly, arises quite naturally in rollout-based planning.

\section{Depth-based Dependence for Rollouts}
\label{sec:depth-baseddependenceforrollouts}

In this section, we propose $X_{\DD}$, a new estimator of value difference whose variance is upper-bounded by that of $X_{\I}$ for all pairs of policies $\pi_{1}, \pi_{2}$, on all MDPs. Further, if $\pi_{1}$ and $\pi_{2}$ share a special structure, then $X_{\DD}$ exploits the structure to achieve a strictly lower variance than $X_{\I}$. Concretely, we say policies $\pi_{1}, \pi_{2}: S \times [H] \to A$ ``agree after $d$ steps'', for some $d \in [H]$, if for \textit{all} $t \in \{d + 1, d + 2, \dots, H\}$ the policies take the same action at every state at time step $t$: that is, for $s \in S$, $\pi_{1}(s, t) = \pi_{2}(s, t)$. The policies are free to take different actions at states when $t \leq d$.

Our difference estimator $X_{\DD}$, given policies $\pi_{1}$ and $\pi_{2}$ that agree after $d$ steps, uses the knowledge of $d$, as follows. For transitions up to (and including) time step $d$, the policies are simulated independently (like in $X_{\I}$). However, for time steps $d + 1$ and after, both policies are evaluated on the same transitions (like in $X_{\D}$). The corresponding backward process to define $X_{\DD}$ works out as follows. Suppose $M = (S, A, P, R, H, s_{1})$. Let $M_{1} \sim M$ have transition function $P_{1}$, and $M_{2} \sim M$ have transition function $P_{2}$. Let $M_{3}$ be the MDP $(S, A, P_{3}, R, H, s_{1})$, where for $(s, a, t) \in S \times A \times [H]$ and $s^{\prime} \in S \cup \{s_{\top}\}$, $$P_{3}(s, a, t, s^{\prime}) = \begin{cases}
P_{2}(s, a, t, s^{\prime}) & \text{if } t \leq d,\\
P_{1}(s, a, t, s^{\prime}) & \text{if } t > d.
\end{cases}$$
In other words, $M_{3}$ uses the same samples as $M_{2}$ for $t \leq d$, and the same samples as $M_{1}$ for $t > d$. For brevity we encode this relation as $M_{3} = M_{2}(1:d) \cdot M_{1}(d + 1:H)$. We have our definition of $X_{\DD}$ (``DD'' since the MDPs on which $\pi_{1}$ and $\pi_{2}$ are evaluated are ``dependent'' after ``depth'' $d$).
\begin{center}
\framebox{\parbox[c]{0.92\textwidth}{
\noindent\underline{$X_{\DD}(M, \pi_{1}, \pi_{2}, d)$.} Sample $M_{1} \sim M$ and $M_{2} \sim M$. Let $M_{3} = M_{2}(1:d) \cdot M_{1}(d + 1:H)$.\\
Return $U(\pi_{1}, M_{1}) - U(\pi_{2}, M_{3})$.
}}
\end{center}
Figures \ref{fig:xi},\ref{fig:xd}, and \ref{fig:xdd} visually convey the difference between $X_{\I}$, $X_{\D}$, and $X_{\DD}$. Observe that policies $\pi_{1}$ and $\pi_{2}$ defined for the MDP in Figure~\ref{fig:counterexample} (in the preceding proof) do not agree for the entire horizon. In this case, $X_{\DD}$ becomes identical to $X_{\I}$. Note, however, that the estimator $X_{\D}$ can have a larger or smaller variance depending on the rewards. In contrast, $X_{\DD}$ ``plays safe'', by assessing the policies' structure. On MDPs in which policies agree after depth $d < H$, and reachable states at $d + 1$ have variance in their returns, $X_{\DD}$ can strictly improve over $X_{\I}$.
\begin{theorem}
\label{thm:xdd}
Let $M = (S, A, T, R, H, s_{1})$ be an arbitrary MDP. Suppose policies $\pi_{1}, \pi_{2}: S \times [H] \to A$ agree after $d$ steps for some $d \in [H]$. Then 
\begin{itemize}[leftmargin=0.8cm]
\setlength{\parindent}{3em}
\item[(i)] $\var(X_{\DD}(M, \pi_{1}, \pi_{2}), d) \leq \var(X_{I}(M, \pi_{1}, \pi_{2})), \text{ and moreover}$
\item[(ii)] the inequality in (i) is strict if $\pi_{1}$ and $\pi_{2}$ both have a positive probability of reaching from starting state $s_{1}$ some state $s_{d+1}$ after $d$-steps, such that the $(H - d)$-step return from $s_{d+1}$ (identical under $\pi_{1}$ and $\pi_{2}$) has strictly positive variance.
\end{itemize}
\end{theorem}
\begin{proof}
We first prove (i). We use ``$M_{3}$'' as shorthand for the MDP $M_{2}(1:d) \cdot M_{1}(d + 1:H)$, which is derived from $M_{1}$ and $M_{2}$. As in the proof of Proposition~\ref{prop:negative}, we begin by splitting $X_{\DD}$ and $X_{\I}$ into variance and covariance terms as in \eqref{eq:varxd} and \eqref{eq:varxi}. For this proof, we have to show that $\cov_{M_{1} \sim M, M_{2} \sim M} 
\left( U(\pi_{1}, M_{1}), U(\pi_{1}, M_{3}) \right) \geq 0$. We prove the following more general claim.
\begin{align}
\underset{M_{1} \sim M, M_{2} \sim M}{\cov} \left( V^{\pi_{1}}_{M_{1}}(s, t), V^{\pi_{2}}_{M_{3}}(\overline{s}, t) \right) 
&\geq 0 \text{ for all } s, \overline{s} \in S \text{ and } t \in [H].
\label{eq:genclaim}
\end{align}
The proof is by induction on $t$. The base case of $t = H + 1$ is trivially true since $V^{\pi_{1}}_{M_{1}}(s, H + 1) = V^{\pi_{2}}_{M_{3}}(\overline{s}, H + 1) = 0$. Now suppose that \eqref{eq:genclaim} is true for $t + 1$. Our proof splits into two cases.

\paragraph{Case 1.} Suppose $t > d$ and $s = \overline{s}$. Since $\pi_{1}$ and $\pi_{2}$ agree after $d$ steps, and since $s = \overline{s}$, we have that $\pi_{1}(s) = \pi_{2}(\overline{s})$. By our design of $X_{\DD}$, transitions in $M_{1}$ and $M_{3}$ are identical for steps $d + 1$ onwards. Hence $V^{\pi_{1}}_{M_{1}}(s, t)$ and $V^{\pi_{2}}_{M_{3}}(\overline{s}, t)$ are the same random variable, and the LHS reduces to $\var_{M_{1} \sim M} V^{\pi_{1}}_{M_{1}}(s, t)$, which must be non-negative.

\paragraph{Case 2.} Consider the complementary case that $t \leq d$ or $s \neq \overline{s}$. Recall the definition $\cov(Y, Z) = \mathbb{E}[YZ] - \mathbb{E}[Y]\mathbb{E}[Z]$. If we apply the Bellman equation in \eqref{eq:bellman} to expand $V^{\pi_{1}}_{M_{1}}(s, t)$ and $V^{\pi_{2}}_{M_{3}}(\overline{s}, t)$, and then cancel out terms containing the reward, we are left with
\begin{align*}
&\underset{M_{1}, M_{2} \sim M}{\cov} \left( V^{\pi_{1}}_{M_{1}}(s, t), V^{\pi_{2}}_{M_{3}}(\overline{s}, t) \right) \\
&= \sum_{s^{\prime}, \overline{s}^{\prime} \in S \cup \{s_{\top}\}} 
\left(
\underset{M_{1}, M_{2} \sim M}{\mathbb{E}} \left[P_{1}(s, \pi_{1}(s, t), t, s^{\prime}) V^{\pi_{1}}_{M_{1}}(s^{\prime}, t + 1) P_{3}(\overline{s}, \pi_{2}(\overline{s}, t), t, \overline{s}^{\prime}) V^{\pi_{2}}_{M_{3}}(\overline{s}^{\prime}, t + 1) \right] \right. \\
& - \left. \underset{M1, M2 \sim M}{\mathbb{E}}\left[P_{1}(s, \pi_{1}(s, t), t, s^{\prime}) V^{\pi_{1}}_{M_{1}}(s^{\prime}, t + 1) \right]
\underset{M_{1}, M_{2} \sim M}{\mathbb{E}}\left[ P_{3}(\overline{s}, \pi_{2}(\overline{s}, t), t, \overline{s}^{\prime}) V^{\pi_{2}}_{M_{3}}(\overline{s}^{\prime}, t + 1) \right] \right).\\
\end{align*}
If $t \leq d$, notice that we have ensured independent samples for evaluating our policies: $\pi_{1}$ on $M_{1}$ and $\pi_{2}$ on $M_{3}$. Samples are anyway independent if $s \neq \overline{s}$, even if $t \geq d$. Thus, $P_{1}(s, \pi_{1}(s, t), t, s^{\prime})$ and $P_{3}(\overline{s}, \pi_{2}(\overline{s}, t), t, \overline{s}^{\prime})$ are independent of each other. They are also independent of $V^{\pi_{1}}_{M_{1}}(s^{\prime}, t + 1)$ and $V^{\pi_{2}}_{M_{3}}(\overline{s}^{\prime}, t + 1)$, which use samples only from $t + 1$ and onwards. Applying this independence and substituting the expected values of the transition probabilities, we obtain
\begin{align*}
&\underset{M_{1} \sim M, M_{2} \sim M}{\cov} \left( V^{\pi_{1}}_{M_{1}}(s, t), V^{\pi_{2}}_{M_{3}}(\overline{s}, t) \right)\\
&=
\sum_{s^{\prime}, \overline{s}^{\prime} \in S \cup \{s_{\top}\}}
P(s, \pi_{1}(s, t), t, s^{\prime}) P(\overline{s}, \pi_{2}(\overline{s}, t), t, \overline{s}^{\prime}) \underset{M_{1}, M_{2} \sim M}{\cov} \left( V^{\pi_{1}}_{M_{1}}(s^{\prime}, t + 1), V^{\pi_{2}}_{M_{3}}(\overline{s}^{\prime}, t + 1) \right).
\end{align*}
By our induction hypothesis, the covariance quantities in the right-hand side are non-negative; since they are scaled by non-negative factors and summed up, the left-hand side must also be non-negative, completing the proof of (i). On the other hand, if even a single covariance term in the right-hand side is strictly positive, and also multiplied by positive probabilities, it would yield a positive left-hand side. Positive covariance originates from the variance terms discussed in case 1: that is, when some state with a positive variance in its return is reached at $t = d + 1$ by both policies. Tracking the progress of positive covariances through our inductive argument delivers the proof of (ii).
\end{proof}
In summary, $X_{\DD}$ dominates $X_{\I}$ in the sense that its variance is never larger on any MDP, and for certain MDPs and sets of policies, its variance is strictly smaller. As indicated by the proof above, the variance that $X_{\DD}$ eliminates is from the returns of states reached after $d$ steps, at which point the two policies take identical decisions. A common setting in which policies do agree after a finite number of steps is simulation-based planning with rollouts, which we consider in the next section. Indeed, our use of the term ``depth'' to describe the number of steps $d$ beyond which the policies agree is in anticipation of this application.

Whereas so far we have primarily used the backward process to define our estimators, we shall furnish an equivalent forward process that is more suited for algorithms. As we proceed, it is also worth bearing in mind that the (full-dependence-) estimator $X_{\D}$ could yet outperform $X_{\DD}$ and $X_{\I}$ on specific MDPs; all we have shown in Proposition~\ref{prop:negative} is that $X_{\D}$ must be worse than $X_{\I}$ (hence also worse than $X_{\DD}$) on some MDPs.

\section{Simulation-based Planning with Rollouts}
\label{sec:simulation-basedplanningwithrollouts}

Planning is the cognitive process by which an agent ``imagines'' the possible trajectories resulting from different behaviours, and selects an action that is calculated to lead to the most promising future. In ``decision-time planning''~\citep[see Chapter 8]{sutton2018reinforcement}, the agent plans towards its action selection at every time step, usually only getting a limited amount of time for computation. Suppose the agent is in state $s$ at time step $t$. A natural formulation of a ``good'' action would be one that leads to maximum state value $V(s, t)$, defined recursively by:
\begin{align}
V(s, t) &= \max_{a \in A}  \left( R(s, a, t) + \sum_{s^{\prime} \in S} P(s, a, t, s^{\prime}) V(s^{\prime}, t + 1) \right).
\label{eq:dtp}
\end{align}
\begin{enumerate}
\item The first challenge in identifying the maximising action is computational. Since the computation branches by a factor of $|A|$ at every reachable state, its tree structure has to be curtailed to some manageable ``depth'' $d$. This leaves the agent with having to \textit{approximate} $V(s^{\prime}, t + d)$. 
\item The second challenge, which we have already discussed, is the non-availability of the probabilities in $P$. Rather, $P$ is only accessible by \textit{sampling}.
\end{enumerate}


Simulation-based planning with rollouts~\citep{bertsekas1997rollout,bertsekas1999rollout,bertsekas2013rollout} is an elegant approach to simultaneously address both challenges. First, $V(s^{\prime}, t + 1)$ is approximated by the value of $(s^{\prime}, t + 1)$ under a ``rollout'' policy $\pi_{\text{R}}$. Rollout policies are quick to implement at each state: for example, they could be rules \citep{bertsekas1999rollout}, or a dot product between state features as weights \citep{silver2016mastering}. Second, although the exact function in \eqref{eq:dtp} involves maximising and averaging, it can still be approximated using simulation. The well-known UCT algorithm \citep{kocsis2006bandit} is a common choice for concentrating the sampling around decisions that are likely to figure in maximising subtrees, and avoiding wasteful simulations.

An alternative to using rollouts is to approximate $V(s^{\prime}, t + d)$ using an evaluation function \citep[see Chapter 3]{russell2010ai}, which can itself be manually tuned \citep{campbell2002deep}, or learned \citep{silver2016mastering}. On the one hand, evaluation functions are usually faster to compute than generating multiple trajectories of rollout policies and averaging the returns. On the other hand, the latter approach may eventually result in better decisions being taken, in part because it yields an unbiased estimate of the rollout policy. The need to balance between computational speed and decision quality stands out in the original AlphaGo program \citep{silver2016mastering}, which approximates 
$V(s^{\prime}, t + d)$ with a convex combination of a state-evaluation and a rollout average. However, subsequent iterations of the program \citep{silver2017mastering, silver2018general} have relied solely on a learned evaluation function.

In the last 1--2 decades, simulation-based planning has become a mainstay of agents operating in a variety of domains. Table~\ref{tab:applications} provides a brief snapshot. Notice that, indeed, many implementations use rollouts. It is in this context that our estimators from the previous sections become relevant.

\begin{table}[t]
\centering
\caption{Representative applications of simulation-based planning.}
\label{tab:applications}
\resizebox{\textwidth}{!}{%
\begin{tabular}{|c|c|c|c|}
\hline
 \textbf{Application} & \textbf{Reference} & \textbf{Uses rollouts} & \textbf{Uses evaluation function} \\ \hline
Real-time strategy games & \cite{balla2009uct} & Yes & No \\ \hline
Go & \cite{silver2016mastering} & Yes & Yes \\ \hline
Chess, Go \& Shogi & \cite{silver2017mastering, silver2018general} & No & Yes \\ \hline
Task \& Motion planning & \cite{paxton2017combining} & Yes & No \\ \hline
Vehicle routing & \cite{mandziuk2017uct} & Yes & No \\ \hline
Chemical synthesis & \cite{segler2018planning} & Yes & No \\ \hline
Autonomous driving & \cite{hoel2019combining} & No & Yes \\ \hline
Task scheduling & \cite{hu2019spear} & Yes & No \\ \hline
Forest management & \cite{neto2020multi} & Yes & No \\ \hline
Trading \& Hedging & \cite{vittori2021monte} & Yes & No \\ \hline
Robot path planning & \cite{dam2022monte} & No & Yes \\ \hline
Fast matrix multiplication & \cite{fawzi2022discovering} & No & Yes \\ \hline
Fast sorting & \cite{mankowitz2023faster} & No & Yes \\ \hline
Path replanning & \cite{trotti2024path} & Yes & No \\ \hline
\end{tabular}%
}
\end{table}
















\subsection{Efficient Forward Process}
\label{subsec:simplan:efficientforwardprocess}

Our formulation of the value difference estimator in sections \ref{sec:estimatingvaluedifference} and \ref{sec:depth-baseddependenceforrollouts} was to select the better of \textit{two} policies, but the same principle generalises to selecting the best from a larger set of policies. Moreover, one usually executes each policy multiple times and averages the returns to cut down variance while estimating value. Concretely, suppose we aim to select the best among $m$ policies after running each policy $n$ times. In the ``independent'' setting, every policy would be evaluated on its own $n$ independently-sampled MDPs, while in the ``dependent'' setting, all $m$ policies would get evaluated on the same set of $n$ independently-sampled MDPs. Meanwhile, if all the policies agree after $d$ steps, the ``depth-dependent'' setting would require each policy to have its own $n$ $d$-horizon MDPs for $d$ steps, then revert to a common set of $n$ ($H - d$)-horizon MDPs for steps $d + 1$ to $H$.

In practice, simulations are performed using (pseudo)random number generators (RNGs), which can be controlled by setting a seed. (1) Initialised with the same seed, an RNG produces the same deterministic outcome; (2) the process generating the outcome for any seed is \textit{independent} of the process emanating from any other seed. Although it is customary to seed RNGs with integers, we may equivalently use strings for seeding, since strings can be uniformly hashed to a range of integers. With this setup, all we have to do to implement independent, dependent, or depth-dependent sampling is to seed each call to the RNG appropriately. The pseudocode below shows the forward process to obtain a random return from a policy under the three sampling regimes; the variable ``simulationIndex'' varies from $1$ to $n$. Policies may be encoded as distinct numbers or state-to-action mappings; including the policy index in the seed is the main difference between the regimes.
\begin{center}
\framebox{\parbox[c]{0.88\textwidth}{
\noindent\underline{Evaluate($M$, $\pi$, $d$, simulationIndex, estimator).}// Assume that $M = (S, A, P, R, H, s_{1})$.\\\\
$s \gets s_{1}$.\\
$\rho \gets 0$.\\
For $t = 1, 2, \dots, H$:\\
\phantom{aaaa}$a \gets \pi(s, t)$.\\
\phantom{aaaa}$\rho \gets \rho + R(s, a)$.\\
\phantom{aaaa}seed = string(s) + string(a) + string(t) + string(simulationIndex).\\
\phantom{aaaa}If (estimator == ``independent'') or (estimator == ``depth-dependent'' and $t \leq d$ ):\\
\phantom{aaaaaaaa}seed = seed + string($\pi$).\\
\phantom{aaaa}$s^{\prime} \gets \text{sample}(P(s, a, t, \cdot), \text{seed})$.\\
\phantom{aaaa}$s \gets s^{\prime}$.\\
Return $\rho$.
}}
\end{center}
The usage of a common seed for different policies makes ours an instance of the general principle of using ``common random numbers'' \citep{Jucker1975POLICYCOMPARINGSE,Glasserman1992SomeGA}. We assume that an outer loop of simulation-based planning invokes the routine above with the appropriate setting of $d$. For fixed $n$ and $m$, we expect the probability of identifying the best policy to be consistent with the variance of our estimators, as analysed in the previous sections. For algorithms such as UCT that explore the $d$-depth search tree dynamically---and hence do not sample policies equally---it is less apparent how estimation and control interact. For now, our assessment of the dynamic setting comes from experiments; we defer theoretical analysis to future work.

\subsection{Empirical Evaluation}
\label{subsec:simplan:empiricalevaluation}

We conduct preliminary experiments on a synthetic MDP with $|S|=7$ states, $|A|=4$ actions, and horizon $H=20$. The transition probabilities $P$ and rewards $R$ are generated uniformly at random from $[0, 1]$, with the former then normalised appropriately.

For our first experiment, we randomly generate a fixed set of $m = 100$ deterministic policies, which all agree after the initial $d=2$ steps. We run each policy $n$ times, and select the policy with the best estimated utility. Figure~\ref{fig:fixed-synthetic} shows the true value of the selected policy as a function of the number of simulated episodes $n$. As predicted, depth-dependent seeding consistently outperforms independent seeing; both methods converge to the optimal value as $n$ increases.

In our second experiment, we implement depth-limited UCT for decision-time planning on the same MDP. At each decision step, UCT is invoked, and the root action with the highest action value is executed. The values of the root actions are estimated using simulations. In each simulation, actions are selected on the basis of UCB values up to depth $d=2$; thereafter, actions are chosen according to a uniform random rollout policy. At each step, the UCT algorithm has to pick a single \textit{action}, rather than a policy, to execute, before replanning for the next step. While planning, the ``policy'' following this action is constantly updated to balance exploration with exploitation up to depth $d$. In our simulations with ``independent'' and ``depth-dependent'' seeding, we pass only the root action, rather than an entire policy, to determine the seed. Although this approach does not find theoretical support, it still results in clear empirical gains, seen in 
Figure~\ref{fig:uct-synthetic}.
Although all three methods perform better with more simulations, the advantage of (depth-)dependence is pronounced when the number of simulations is small, inevitable when the agent's ``think time'' is limited \citep{zilberstein1996using}.

In both our experiments, we notice that the dependent seeding scheme outperforms the independent variant and is indistinguishable from the depth-dependent variant. This is not surprising. However, to be sure, we also compare the schemes for selecting between $\pi_{1}$ and $\pi_{2}$ in the MDP constructed in the proof of Proposition~\ref{prop:negative}, taking care to set rewards to achieve a negative covariance under $X_{\D}$. Since $\pi_{1}$ and $\pi_{2}$ do not agree for the entire horizon of $H = 2$, the independent and depth-dependent variants are identical. As expected, we observe that the dependent variant performs worse.

We achieve consistent results in our experiments (not shown here) when we vary $|S|$, $|A|$, $H$, and $d$. We now progress from our synthetic setup to applications of interest to practitioners. The applications presented in the next two sections both require planning under significant stochasticity.


\begin{figure}[b]
    \centering
    
    \begin{subfigure}[t]{0.32\textwidth}
        \centering
        \includegraphics[width=\linewidth]{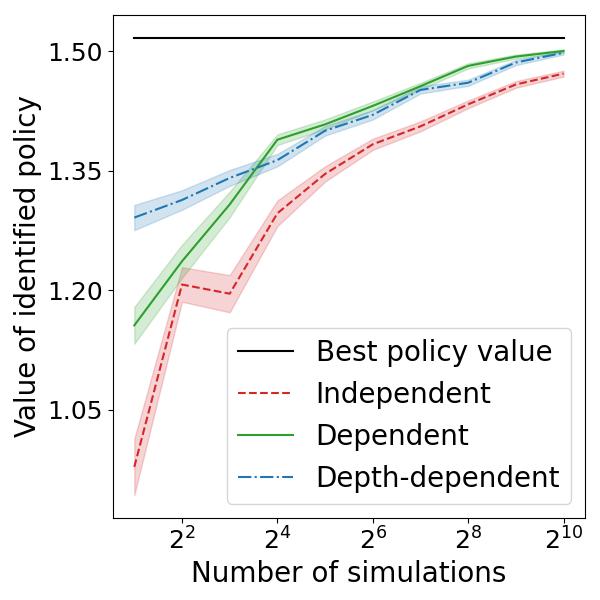}
        \caption{Synthetic MDP: Fixed policies.}
        \label{fig:fixed-synthetic}
    \end{subfigure}
    \hfill
    \begin{subfigure}[t]{0.32\textwidth}
        \centering
        \includegraphics[width=\linewidth]{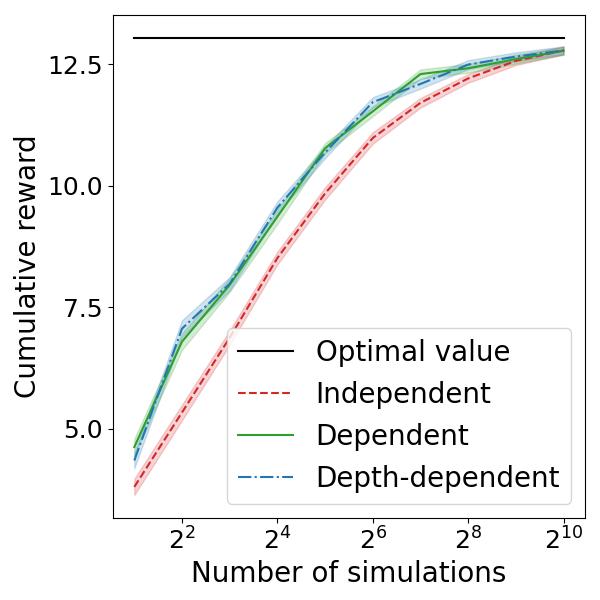}
        \caption{Synthetic MDP: UCT.}
        \label{fig:uct-synthetic}
    \end{subfigure}
    \hfill
    \begin{subfigure}[t]{0.32\textwidth}
        \centering
        \includegraphics[width=\linewidth]{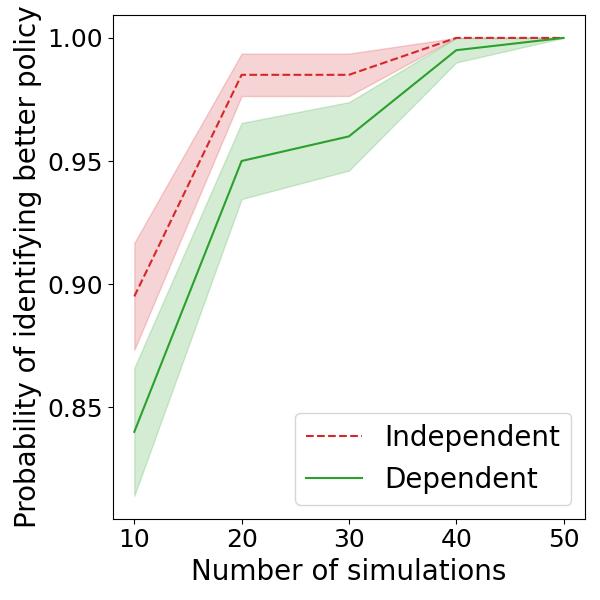}
        \caption{MDP from Figure~\ref{fig:counterexample}: $\pi_{1}, \pi_{2}$.}
        \label{fig:counter-results}
    \end{subfigure}
    
    \caption{Performance metrics against the number of simulations on synthetic tasks. Results (here and in Figure \ref{fig:ftvaf_res}) are averages from $200$ runs, and show one standard error. In Figure (c), the rewards are set to 
    $r_0=2, r_1=4, r_2=3, r_3=2$, which gives $X_{\D}$ a larger variance than $X_{\I}$. }
    \label{fig:syn_results}
\end{figure}

\section{Payouts for a Fixed-Term Variable Annuity Fund}
\label{sec:ftvaf}


A fixed-term variable annuity fund (FTVAF)~\citep{bacinello2011variable} is a financial product that enables retirees to invest in high-yield, high-risk assets such as stocks, which are typically excluded from conventional retirement schemes due to their volatility. In this collective investment scheme, participants contribute to a pooled fund and, in return, receive post-retirement annual payments that vary based on the performance of the underlying investments. Unlike traditional lifelong annuity products, which provide fixed payments, FTVAF operates over a predetermined, fixed \textit{term}, and features variable annual payments. Variable annuities offer the potential for higher payouts.
Since the fund targets retirees, there is a significant probability that some members might pass away during the fund's term. In such scenarios, their beneficiaries receive a proportional sum of the corpus. The challenge in FTVAF management is to negotiate its significantly stochastic environment, made so both by market volatility~\citep{bacinello2011variable} and by mortality~\citep{aalen1992modelling}.

Managing an FTVAF is a sequential decision-making problem where, every year, we are required to choose how much to pay as annuity to the members of the fund. The primary objective in FTVAF management is to maximise the total annuity received by members while they are alive, rather than maximising the residual wealth passed on to beneficiaries. When a member dies, their share of the fund ($\frac{1}{N}$ of the total corpus, if there are $N$ members) is distributed to their beneficiaries, but this does not factor into the main optimisation goal. The core challenge is to determine an optimal strategy for distributing the initial corpus $W_1$ as annuities over the fund’s lifespan, accounting for the randomness in both investment returns and member mortality.
A compact snapshot of the problem is given in Figure \ref{fig:ftvaf_flow}.

\begin{figure}[b]
    \centering
    \begin{subfigure}[t]{0.57\textwidth}
        \centering
        \includegraphics[width=\linewidth]{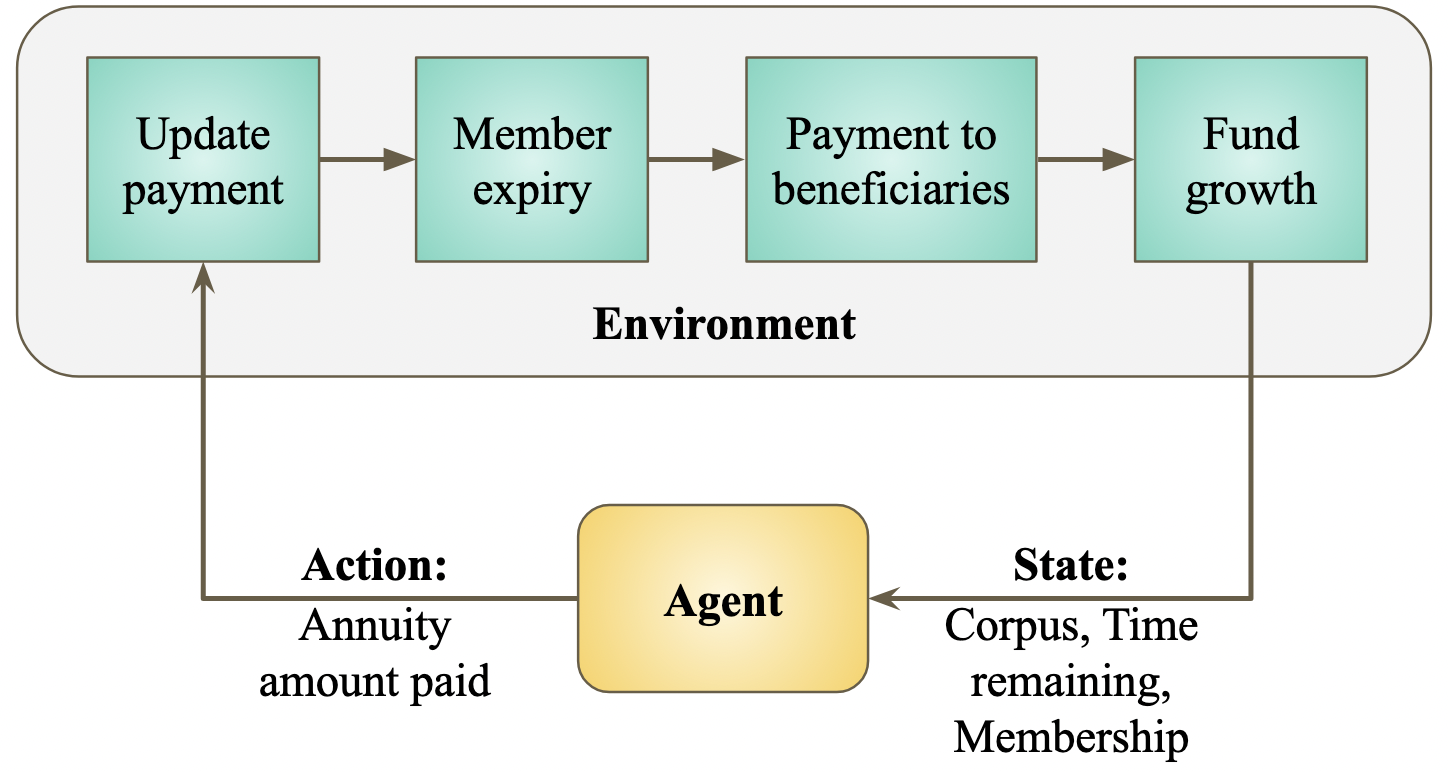}
        \caption{Logical flow of the FTVAF environment.}
        \label{fig:ftvaf_flow}
    \end{subfigure}
    \hfill
    \begin{subfigure}[t]{0.4\textwidth}
        \centering
        \includegraphics[width=\linewidth]{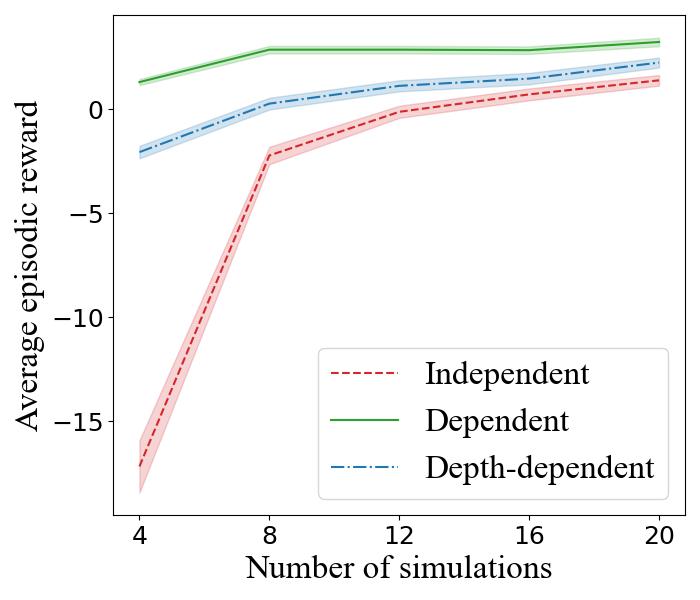}
        \caption{Average episodic rewards.}
        \label{fig:ftvaf_res}
    \end{subfigure}
    
    \caption{Figure (a) explains the sequence of steps in the FTVAF task, while Figure (b) records the performance of different seeding techniques.}
    \label{fig:ftvaf}
\end{figure}

\paragraph{Problem.}
The MDP $(S, A, P, R, H, s_1)$ for the FTVAF problem is expanded as follows. Without loss of generality, we start with an initial corpus of $W_1=1$, since all transitions are proportional.

\begin{itemize}
    \item Each state is specified by $s = (W_h,h,l_h,L_{h,d})$ where $W_h$ is the current corpus amount, $h$ is the elapsed number of years, $l_h$ is the fraction of the original member set who are still alive, and $L_{h,d}=\{L_{h}(x)\}$ is a vector containing the proportion of people of age $x$ in the member set.
    
    \item Action $a_h$ is a scalar in the range $[0,1]$, where $0$ represents ``no payout'' and $1$ represents ``entire remaining corpus paid out''.

    \item Each variable of the state evolves, with an increment in $h$, as follows.
        \begin{itemize}
            \item For the subset of members of age $x$ alive at time $h$, the number remaining in the next year
            $$L_{h+1}(x+1)=L_h(x)\,\left( 1-\frac{\mathrm{Poisson}(N_x\lambda_x)}{N_x} \right),$$
            where $\lambda_x$ is the Poisson rate for the particular age $x$ and $N_x$ is the current number of members in this age group. Simulating mortality using a Poisson process is a standard approach~\citep{aalen1992modelling}. We set the parameter $\lambda_x$ based on the 2003 US vital statistics report~\citep{Arias2006USLifeTables}.
            \item The revised fraction of people alive is given by $l_{h+1}=\sum_x L_{h+1}(x)$.
            \item $W_{h+1}$ is computed from $W_h$ by first deducting the amount paid as annuity in the current year, followed by a proportional payment to beneficiaries of expired members, and then accounting for the growth of the remaining wealth. For the latter, we use the standard model of Geometric Brownian Motion \citep{samuelson1973mathematics}, setting parameters of $\mu=0.15$ and $\sigma=0.2$. 
            Using ``$Z$'' to denote a a standard Gaussian variable, the overall corpus transition is thus $$W_{h+1} = W_h (1 - a_h)\cdot \left(\frac{l_{h+1}}{l_h}\right) \cdot e^{(\mu - \frac{\sigma^2}{2}) + \sigma Z}.$$
        \end{itemize}
    \item The reward $R_h$ at each time step is the total amount paid out as annuity in the current year: that is, $W_h\,a_h$. Additionally, we provide terminal rewards for the two ways in which a simulation path can terminate,
        \begin{itemize}
            \item A large penalty of $l_{h}\,(h-(H+1))\,C_1$ in case the corpus drops below $5\%$ of the original amount, relative to the fraction of live subscribers, before the end of tenure. This penalty encourages the agent to keep the fund solvent throughout the tenure. $C_1$ is set as 5.
            \item A smaller penalty of $W_H\,C_2$ to discourage the agent from keeping too much wealth in the corpus up to the end of the tenure. $C_2$ is set as 0.03.
        \end{itemize}
\end{itemize}

\paragraph{Solution.}
We define a simulation-based planning approach that satisfies the assumptions in Section \ref{sec:depth-baseddependenceforrollouts}, specifically that the policies being considered vary only in the first $d=1$ step. This amounts to single-step lookahead, which is commonly used in applications \citep{bertsekas1999rolloutsch,Doshi+TAKK:2024}. We divide the action space $A$ into 101 steps from $0\%$ of the original amount to $100\%$ of the original amount. Then we evaluate all 101 branches (values of $a_h$) at the current time, followed by a simple rollout policy that annually pays out a constant $11\%$ of the remaining corpus. The action $a_h$ with the highest average reward is then chosen for each decision step. The FTVAF environment is such that the random growth and mortality \textit{rates} do not depend on the state and action taken. We only have to set predetermined sequences of samples from the standard Gaussian variables $Z$ and the mortality outcomes for each year. In the dependent case, the $i^{th}$ simulation of each root action is evaluated on the same sequence, while in the independent case, the $i^{th}$ simulation of each root action is evaluated on a different sequence. Under depth-dependence, only the very first transition of any root action gets its own, independent sample.


Figure~\ref{fig:ftvaf_res} compares the episodic reward achieved under the three sampling regimes, when the horizon is fixed to $H = 20$ years. The simulation also requires an initial population of subscribers; we draw the ages of these members uniformly at random from $[60, 70]$ years.
The plot demonstrates that dependent and depth-dependent schemes are more effective, especially when the number of simulations is small. Given the high stochasticity and the large number of actions, independence even in just the first time step appears to hurt performance.

\section{Ludo}
\label{sec:ludo}

Ludo is a strategy-based board game for 2--4 players. As shown in Figure~\ref{fig:ludo_board}, each player has 4 pieces, distinguished by colour from other players' pieces. All 4 pieces of a player begin in a designated starting area. The players then take turns to advance their pieces through a circuit on the ludo board; the first player to get all 4 of its pieces home is declared the winner. In general, pieces move on the board according to the outcome of a die roll. For example, if a 3 comes up from a player's roll, the player must move one of its eligible pieces by 3 squares. The decision to make is which piece to advance. There are several special cases in the game, such as repeated moves upon a ``6'' outcome, sending an opponent's piece back to the start by landing on the same square, and avoiding this very event by occupying designated ``safe'' squares. We refer the reader to \citet{ludopy} (creator of ``LUDOpy'') for a detailed description of the rules of the game.

We consider a 2-player game in which an agent plays against an opponent that takes random actions.
We use the LUDOpy~\citep{ludopy} simulator and model the opponent as part of the environment. In this formulation, there are two main sources of stochasticity: die rolls and opponent moves.

\paragraph{Problem.}

The MDP $(S, A, P, R, H, s_1)$ for Ludo is formulated as follows.


\begin{itemize}
    \item The state space is given by $S = (\delta, p_1, p_2, p_3, p_4, \hat{p}_1, \hat{p}_2, \hat{p}_3, \hat{p}_4)$, where $1\leq\delta\leq6$ denotes the outcome of die roll, and for $1\leq i \leq 4$, $p_i$ and $\hat{p}_i$ denote the board position of $i^{th}$ piece of agent and opponent, respectively. This leads to $\approx 10^{11}$ states~\citep{alvi2011complexity}.
    
    \item The action space for agent is given by $A = \{a_1,a_2,a_3,a_4\}$ where executing $a_i$ means moving the $i^{th}$ piece according to die outcome $\delta$. Note that in some states, not all four actions are available.

    \item For a given (state, action) pair, state transitions depend on three components: game rules, the outcome of the die roll, and the opponent's move. Given a state, the chosen action of the agent is executed on the basis of the outcome $\delta$. If $\delta \neq 6$, the die is rolled for the opponent, which then makes its move. The LUDOpy simulator automatically implements exceptions such as repeated moves on getting a 6.
Finally, the die is rolled to obtain the agent's next state. 
    
    \item The agent is given a reward of 1 if it wins the game and 0 otherwise.
    
    \item The actual length of each game is variable, but we seldom see games exceeding $300$ moves.
    
    \item $s_1$ is the designated start state of the game when all pieces are in their respective starting area.
\end{itemize}

\begin{figure*}[t]
    \centering
    \begin{subfigure}[t]{0.43\textwidth}
        \centering
        \includegraphics[width=\linewidth]{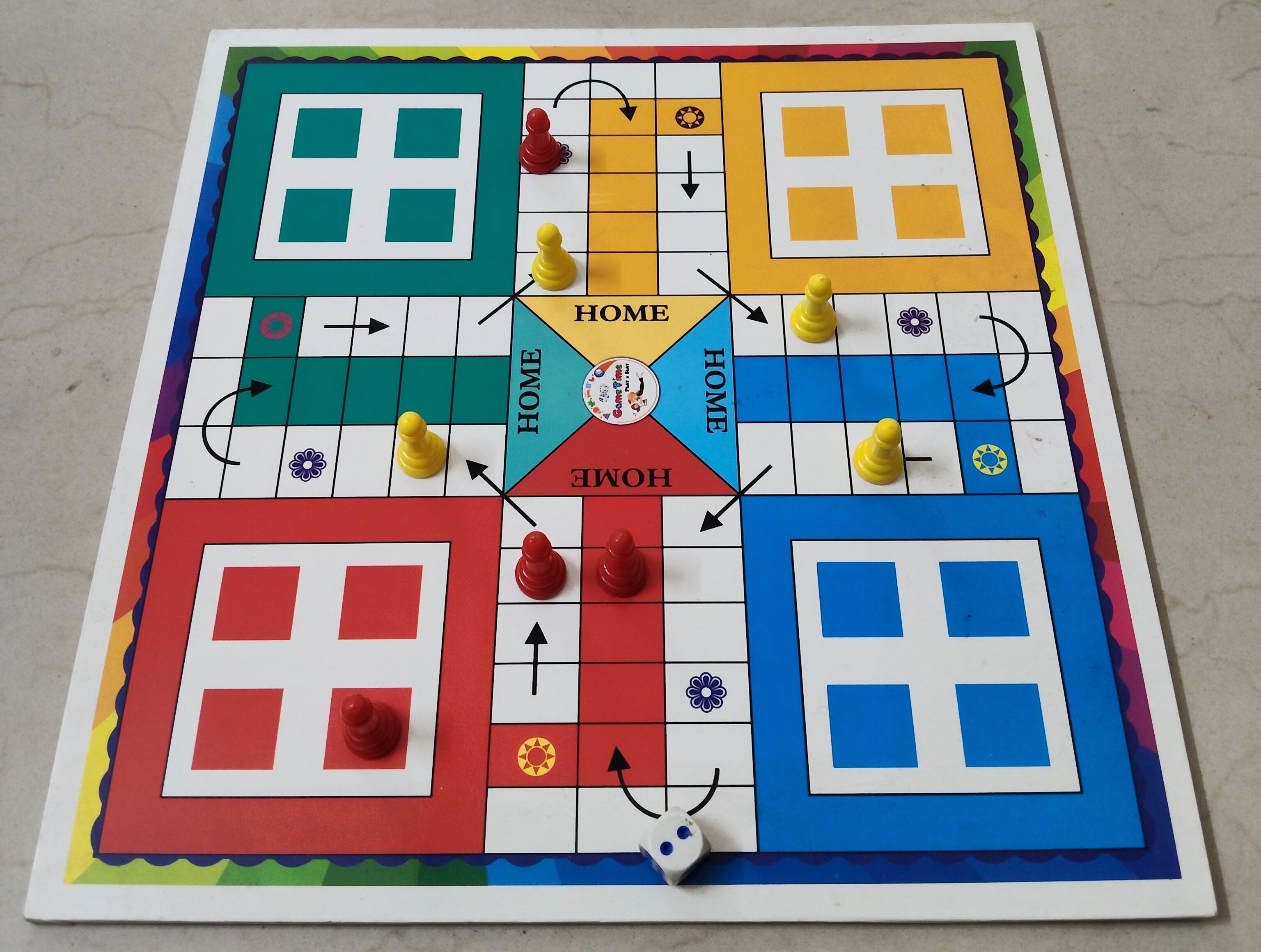}
        \caption{Ludo board}
        \label{fig:ludo_board}
    \end{subfigure}
    \hspace{0.1\columnwidth}
    \begin{subfigure}[t]{0.4\textwidth}
        \centering
        \includegraphics[width=\linewidth]{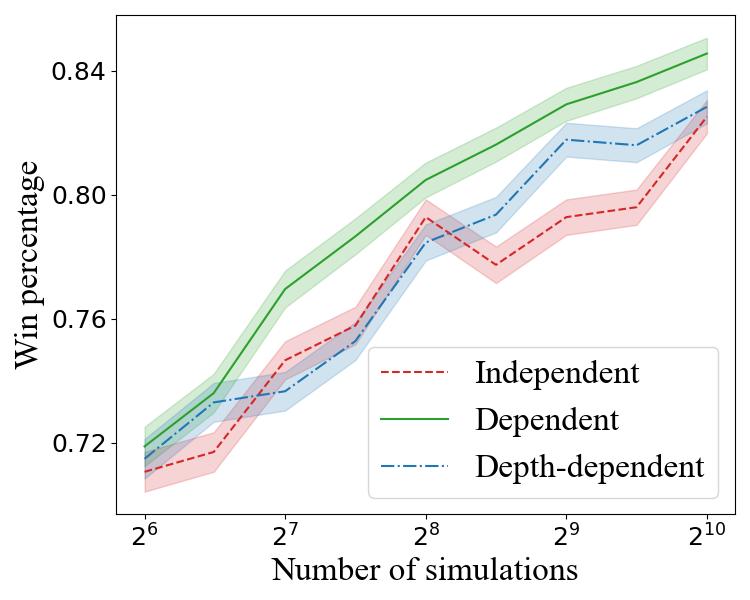}
        \caption{Win percentage (across 5000 games).}
        \label{fig:ludo_res}
    \end{subfigure}
    
    \caption{Ludo: Figure (a) shows the board, and Figure (b) the performance of simulation-based planning against an opponent that selects from available actions uniformly at random.}
    \vspace{-0.5cm}
    \label{fig:ludo}
\end{figure*}

\paragraph{Solution.}
We implement an agent that uses depth-limited UCT for decision time planning against an opponent making moves uniformly at random. At each decision step, where at least 2 valid actions are available, the agent invokes UCT and takes the action with the highest utility, as estimated using simulations. The search tree comprises decision nodes for the agent's actions and chance nodes for die outcomes. In each simulation, actions are selected on the basis of UCB values up to depth $d=2$, and thereafter actions are chosen uniformly at random (this is the rollout phase).
We use predetermined sequences of die outcomes and opponent moves for each round. As before, the $j^{th}$ simulation of each root action is evaluated on the same sequence in the dependent setting (or depth-dependent beyond $d$), otherwise on an independently-sampled sequence.

To blunt the high stochasticity in Ludo games, 
we run 5000 games between the agent and opponent, with the opponent always being the first player.
Figure \ref{fig:ludo_res} reports the win percentage as the number of simulations is varied. The dependent seeding scheme clearly dominates, highlighting that a simple seeding trick can provide a significant performance boost for UCT. The less-convincing performance of the depth-dependent scheme motivates an expansion of our theoretical analysis to UCT.

\section{Related Work}
\label{sec:literature}


Low variance is a desirable property for simulation-based stochastic estimation procedures.
A straightforward approach to reduce variance is to simply use more samples or simulations. However, $k^2$ extra simulations are needed to reduce the standard error by a factor of $k$. Thus, this approach is computationally expensive, sometimes even infeasible when compute time is limited.
Over the years, researchers have come up with variance-reduction techniques for various stochastic estimation~\citep{james1985variance, mcgeoch1992analyzing}. Prominent approaches include control variates~\citep{lavenberg1981perspective, glynn2002some}, antithetic variates~\citep{hammersley1956new, cheng1982use}, common random numbers~\citep{Jucker1975POLICYCOMPARINGSE,Glasserman1992SomeGA}, quasi-Monte Carlo methods~\citep{niederreiter1992random, hung2024review}, stratified sampling and importance sampling~\citep[see Chapter 4]{glasserman2004monte}, conditioning~\citep[see Chapter 8]{brandimarte2014handbook}, and moment matching~\citep[see Chapter 10]{jackel2002monte}.




Our work focuses on variance reduction using common random numbers (CRN). The idea of using CRN was originally proposed from a control theoretic perspective \citep{Jucker1975POLICYCOMPARINGSE, Glasserman1992SomeGA, Spall_2003}. It was used to improve the comparison between stochastic simulation outcomes and is broadly equivalent to the practice of setting the seeds of RNGs in reinforcement learning and planning. The primary reference for the application of CRN in MDPs is the thesis of \cite{Ngthesis}. The main related algorithmic contribution is PEGASUS (Policy Evaluation of Goodness and Search Using Scenarios) ~\citep{Ng2000PEGASUSAP}, designed for efficient policy search in MDPs and Partially Observable MDPs (POMDPs). The authors illustrate the perspective that by using CRN, a stochastic MDP can effectively be treated as a distribution of deterministic MDPs. This same view explains the backward process that we described in sections \ref{sec:estimatingvaluedifference} and \ref{sec:depth-baseddependenceforrollouts}. However, the focus of our theoretical analysis is somewhat different from that of PEGASUS~\citep{Ng2000PEGASUSAP}. The latter argues that the optimiser of an \textit{infinite} set of policies cannot be found by simulation-based evaluation unless the simulations are dependent, in the manner of our $X_{\D}$. Our own focus has been to improve the winner-identification process on a specific family of policies (which agree beyond some depth) by reducing variance in their evaluations. Interestingly, other theoretical studies have used CRN as a means to develop sample-efficient PAC reinforcement learning algorithms \citep{kearns1999approximate, Kalyanakrishnan+SG:2025}. 

The most well-known application of PEGASUS is its use in a model-based reinforcement learning algorithm for autonomous helicopter flight~\citep{A.Ng2003}. The concept of CRN has also been used in robotics~\citep{Wang_2010} and disease simulation modelling \citep{stout2008keeping}. Implementations of many popular optimisation algorithms---such as Trust Region Policy Optimisation~\citep{schulman2015trust} and Evolutionary Strategies \citep{Salimans2017EvolutionSA}---also incorporate CRN into their workflows. As our own demonstrations on FTVAF and Ludo reinforce, the concept of CRN is simple to implement and can offer significant performance gains.



\section{Conclusion and Future Work} \label{sec:conclusion}

Simulation-based planning 
is widely used by agents operating in stochastic environments. Implementations invariably rely on RNGs to generate pseudorandom samples. The question motivating our work is whether planning can be made more efficient by deliberately controlling the seed of the underlying RNG. We provide an affirmative response, along with theoretical and empirical support. All our proposal requires is that the seed for each sampling operation be set as a function of certain variables (such as the state, action, time step, simulation number, and the policy or root action), which are typically quite easy to store and update during simulations.

Our theoretical analysis shows that an ``obvious'' way of seeding, by which the policies being compared are evaluated on the same set of MDPs, can potentially lead to a higher variance in the value difference estimate. However, this anomaly appears to require a certain interaction between the policies and the underlying task, which is quite uncommon in practice, including in our own evaluations on two real-world applications. We also develop a ``depth-dependent'' alternative that is safer in theory, and remains quite effective in practice.

Our current analysis stops at upper-bounding the variance of our proposed estimator, and does not explicitly incorporate this property into sample-complexity bounds. Also, our simplification does not model dynamic algorithms such as UCT, which do not explore policies uniformly. These are avenues for future work. On the practical side, extending evaluations to an even broader range of tasks could provide pointers for refining our framework. We publish code to replicate all our experiments (submitted for review) and to enable  researchers to implement CRN in other applications.

\bibliography{main}
\bibliographystyle{rlj}

\end{document}